\documentclass[final,3p]{elsarticle}
\usepackage{geometry}
\usepackage{mathrsfs,enumerate,amssymb,amsmath,bm,amsthm,color,lineno}
\usepackage[all]{xy}
\usepackage{latexsym}
\usepackage{amscd,verbatim}
\usepackage{graphicx}
\usepackage[colorlinks=true]{hyperref}
\usepackage{amssymb}
\usepackage{placeins}
\usepackage{float}
\usepackage{varioref}
\usepackage{pifont}

\theoremstyle{plain}
\newtheorem{theorem}{Theorem}[section]

\newtheorem{lemma}{Lemma}[section]

\theoremstyle{definition}
\newtheorem{definition}{Definition}[section]

\newtheorem{remark}{Remark}
\newtheorem{claim}{~~~~Claim}

\numberwithin{equation}{section}

\journal{Soft Computing}

\begin{document}

\begin{frontmatter}
\title{On the algebraic structures of the space of interval-valued intuitionistic fuzzy numbers\tnoteref{mytitlenote}}
\tnotetext[mytitlenote]{This work was supported by the National Natural Science Foundation of China
(Nos.~11601449 and 71771140), the Key Natural Science Foundation of Universities in Guangdong Province
(No.~2019KZDXM027), and the Youth Science and Technology Innovation Team of Southwest Petroleum University for Nonlinear
Systems (No.~2017CXTD02).}


\author[a2,a3]{Xinxing Wu\corref{mycorrespondingauthor}}
\cortext[mycorrespondingauthor]{Corresponding author}
%\address[a1]{School of Mathematics and Statistics, Guizhou University of Finance and
%Economics, Guiyang 550025, China}
\address[a2]{School of Sciences, Southwest Petroleum University, Chengdu, Sichuan 610500, China}
\address[a3]{Zhuhai College of Jilin University, Zhuhai, Guangdong 519041, China}
\ead{wuxinxing5201314@163.com}

\author[a2]{Chaoyue Tan}
\ead{q520713277@163.com}

\author[a4]{G\"{u}l Deniz \c{C}ayl{\i}}
\address[a4]{Department of Mathematics, Faculty of Science, Karadeniz Technical University, 61080 Trabzon, Turkey}
\ead{guldeniz.cayli@ktu.edu.tr}

\author[a5]{Peide Liu}
\address[a5]{School of Management Science and Engineering, Shandong University
of Finance and Economics, Jinan, Shandong 250014, China}
\ead{peide.liu@gmail.com}

\begin{abstract}
This study is inspired by those of Huang et al.~(Soft Comput. 25,
2513--2520, 2021) and Wang et al.~(Inf. Sci. 179, 3026--3040, 2009) in
which some ranking techniques for interval-valued intuitionistic fuzzy
numbers (IVIFNs) were introduced. In this study, we prove that
the space of all IVIFNs with the relation in the method for comparing
any two IVIFNs based on a score function and three types of entropy
functions is a complete chain and obtain that this relation is an
admissible order. Moreover, we demonstrate that IVIFNs are
complete chains to the relation in the comparison method for
IVIFNs on the basis of score, accuracy, membership uncertainty index, and
hesitation uncertainty index functions.
\end{abstract}
\begin{keyword}
Interval-valued intuitionistic fuzzy numbers~(IVIFNs); Complete lattice; Entropy function;
Score function.
\end{keyword}
\end{frontmatter}

\section{Introduction}

Zadeh~\cite{zadeh} proposed the notion of fuzzy sets (FSs) characterized by
a membership function, which appoints to each object a grade of membership
ranging between zero and one. Since Zadeh's efficient study, some authors
presented the generalization of the FSs, such as the concepts of type-1
fuzzy sets (T1FSs), interval type-2 fuzzy sets (IT2FSs), generalized type-2
fuzzy sets (GT2FSs) (see {\cite{Ata2017,CM2008,C2014,G2017,P2016}). This
generalization} has induced substantial theoretical improvements and several
applications \cite{K2004,Z2005}, and ensures a significant alternative other
than probability theory to characterize uncertainty, imprecision, and
vagueness in numerous areas \cite{Z2008}.

Atanassov \cite{Ata1986} introduced the intuitionistic fuzzy sets (IFSs) as
a natural extension of Zadeh's FSs, also called the Atanassov's
intuitionistic fuzzy sets (AIFSs). In the theory of IFSs, both the degree of
membership and the degree of non-membership are given to each element of the
universe, and their sum is smaller than or equal to $1$. Hence, IFSs are
more effective than FSs in handling with the uncertainty and fuzziness of
systems. Atanasov \cite{Ata2017} discussed type-1 fuzzy sets (T1FSs) that
form the base of FSs extensions. Furthermore, a comparison between T1FSs and
IFSs was presented in \cite{Ata2017}, and their some applications were
studied in \cite{AR2018,S2006}. Atanassov and Gargov \cite{AG1989} let the
membership and nonmembership functions be interval values rather than single
values. Thereby, they further generalized IFSs to the situation of
interval-valued intuitionistic fuzzy sets (IVIFSs), which have been
efficaciously used in the fields of supply and investment decision-making
\cite{TC2017,ZT2020}. Atanasov \cite{Ata1994,Ata2020} also defined some different
functional laws of IVIFSs.

In recent times, Xu~\cite{X2007} applied the score function and accuracy
function to compare interval-valued intuitionistic fuzzy numbers (IVIFNs).
Subsequently, Ye \cite{ye2009}, Nayagam et al.~\cite{N2011}, Sahin~\cite%
{SA2016}, Zhang and Xu~\cite{ZX2017} and Nayagam et al.~\cite{N2017}
provided some additional accuracy functions to compare IVIFNs and developed
multicriteria fuzzy decision-making techniques. It is well known that these
approaches cannot identify the difference between two arbitrary IVIFNs in
some situations due to the particular properties of intervals. For this
reason, to explore the difference between two IVIFNs, Wang et al.~\cite%
{WLW2009} proposed two new functions: the membership uncertainty index and
the hesitation uncertainty index. Furthermore, a complete ranking method for
two arbitrary IVIFNs using these functions was developed. Later, Huang et
al.~\cite{HZX2021} introduced a novel comparison technique for IVIFNs, which
tells the difference between any two IVIFNs. Their method exploits a score
function and three types of entropy functions, where the entropies are
crucial concepts for measuring the fuzziness of uncertain information.

This work is motivated by those of Huang et al.~\cite{HZX2021} and Wang et
al.~\cite{WLW2009} in which more functional and more reasonable techniques
for comparing two IVIFNs were proposed. The main of this paper is to enhance
some results presented in the mentioned above works by providing new
theorems related to the methodologies to compare two IVIFNs. In this paper,
considering a score function and three types of entropy functions, we
demonstrate that IVIFNs with the order in the procedure to compare any two
IVIFNs proposed by~\cite{HZX2021}, denoted by ``$\leq _{_{\mathrm{HZX}}}$", are
complete totally order sets. Furthermore, we observe that the order $\leq _{_{\mathrm{HZX}%
}}$ is an admissible order on IVIFNs. Moreover, taking into account a
prioritized sequence of score, accuracy, membership uncertainty index, and
hesitation uncertainty index functions, we show that IVIFNs are complete
totally order sets to the order in the ranking principle for IVIFNs given by
\cite{WLW2009}, denoted by $\leq _{_{\mathrm{WLW}}}$.

The rest of this study is organized as follows. Section~\ref{Sec-2} reviews some basic
concepts related to IVIFSs and some classical score functions and accuracy
functions on them. In Section~\ref{Sec-3}, we reveal that IVIFNs are complete totally
order sets, regarding the order $\leq _{_{\mathrm{HZX}}}$ based on a score
function and three types of entropy functions and prove that $\leq _{_{\mathrm{HZX}}}$
is an admissible order on IVIFNs. In Section~\ref{Sec-4}, we propose that IVIFNs with
the order $\leq _{_{\mathrm{WLW}}}$ on the basis of score, accuracy,
membership uncertainty index, and hesitation uncertainty index functions are
complete totally order sets. Section~\ref{Sec-5} provides concluding remarks of this
research.

\section{Preliminaries}\label{Sec-2}

\subsection{Lattices}

In the following, some basic notions and results about lattices are
recalled. These terms will be used in the sequel.

\begin{definition}[\textrm{\protect\cite{Bir1967}}]
A binary relation $\preceq $ defined on a non-empty set $L$ is called a
\textit{partial order} on the set $L$ if, for all $a,b,c\in L,$ it satisfies the
following properties:
\begin{enumerate}[(1)]
\item (Reflexivity) $a\preceq a$.

\item (Antisymmetry) If $a\preceq b$ and $b\preceq a$, then $a=b$.

\item (Transitivity) If $a\preceq b$ and $b\preceq c$, then $a\preceq c$.
\end{enumerate}
\end{definition}

If, for all $a,b\in L$, either $a\preceq b$ or $b\preceq a$, then we say $%
\preceq $ is a \textit{total order} on $L$. A non-empty set with a partial order on it
is called a \textit{partially ordered set}, or more briefly a \textit{poset}. And if the
relation is a total order, then we say it a \textit{totally ordered set} or simply a
\textit{chain}.

A poset $(L,\preceq ,\bm{0},\bm{1})$ is \textit{bounded} if it has top and bottom
elements, which are denoted as $\bm{1}$ and $\bm{0}$, respectively; that is,
two elements $\bm{1}$, $\bm{0}\in L$ exist such that $\bm{0}\preceq a\preceq %
\bm{1}$ for all $a\in L$.

Let $(L,\preceq )$ be a poset and $X\subset L$. An element $u\in L$ is said
to be an \textit{upper bound} of $X$ if $x\preceq u$ for all $x\in X$. An
upper bound $u$ of $X$ is said to be its \textit{smallest upper bound} or
\textit{supremum}, written as $\bigvee X$ or $\sup X$, if $u\preceq y$ for
each upper bound $y$ of $X$. Dually, we can define the \textit{greatest
lower bound} or \textit{infimum} of $X$, written as $\bigwedge X$ or $\inf X$%
. In the case of pairs of elements, it is customary to write
\begin{equation}
x\vee y=\sup \{x,y\}\text{ and }x\wedge y=\inf \{x,y\}.
\label{sup-inf-operation}
\end{equation}

\begin{definition}[\textrm{\protect\cite{Bir1967}}]
\label{Lattice-Def-1} A \textit{lattice }$L$ is a poset in which for all $%
a,b\in L$ the set $\{a,b\}$ has a supremum and an infimum. If there are
elements $\bm{0}$ and $\bm{1}$ in $L$ such that $\bm{0}\preceq a\preceq %
\bm{1}$ for all $a\in L$, then $L$ is called a \textit{bounded lattice}.
\end{definition}

\begin{definition}[{\textrm{\protect\cite[Definition~O-2.1]{GHKLMS2003}}}]
A lattice $L$ is called \textit{complete} if, for any subset $A$ of $L$, the
greatest lower bound and the smallest upper bound of $A$ exist. A totally
ordered complete lattice is called a \textit{complete chain}.
\end{definition}

\begin{lemma}[\textrm{\protect\cite{Bir1967}}]
\label{Complete-Char} Let $L$ be a bounded lattice. Then the following
statements are equivalent:
\begin{enumerate}[{\rm (i)}]
\item $L$ is a complete lattice;

\item Every nonempty subset of $L$ has an infimum;

\item Every nonempty subset of $L$ has a supremum.
\end{enumerate}
\end{lemma}

\subsection{Interval-valued intuitionistic fuzzy sets}

In the following, we present classical concepts of intuitionistic fuzzy sets
(IFSs) and interval-valued intuitionistic fuzzy sets (IVIFSs) as well as
classical score functions and accuracy functions on IVIFSs to be helpful
future discussions.

\begin{definition}[{\textrm{\protect\cite[Definition 2.3]{Ata1999}}}]
Let $X$ be the universe of discourse. An \textit{interval-valued
intuitionistic fuzzy set} (IVIFS) $A$ in $X$ is defined as an object having
the following form
\begin{equation}
A=\left\{ \langle x,\mu _{A}(x),\nu _{A}(x)\rangle \mid x\in X\right\} ,
\label{eq-IVIFS-1}
\end{equation}%
where $\mu _{A}(x)=[\mu _{A}^{L}(x),\mu _{A}^{R}(x)]$ and $\nu _{A}(x)=[\nu
_{A}^{L}(x),\nu _{A}^{R}(x)]$ are subintervals of $[0,1]$ denoting the
\textit{membership degree} and \textit{non-membership degree} of the element
$x$ to the set $A$, respectively, which meet the condition of $\mu
_{A}^{R}(x)+\nu _{A}^{R}(x)\leq 1$.
\end{definition}

Clearly, if $\mu_{A}^{L}(x)=\mu_{A}^{R}(x)$ and $\nu_{A}^{L}(x)=%
\nu_{A}^{R}(x)$, then an IVIFS reduces to a traditional IFS.

For an IVIFS $A$, the \textit{indeterminacy degree} $\pi _{A}(x)$ of element
$x$ belonging to the IFS $A$ is defined by $\pi _{A}(x)=[1-\mu
_{A}^{R}(x)-\nu _{A}^{R}(x),1-\mu _{A}^{L}(x)-\nu _{A}^{L}(x)]$. In \cite%
{XC2012}, the pair $\langle \mu _{A}(x),\nu _{A}(x)\rangle $ is called an
\textit{interval-valued intuitionistic fuzzy number} (IVIFN), which was also
called an \textit{interval valued intuitionistic fuzzy pair} (IVIFP) in \cite%
{Ata2020}. For convenience, we use $\alpha =\langle \lbrack \mu _{\alpha
}^{L},\mu _{\alpha }^{R}],[\nu _{\alpha }^{L},\nu _{\alpha }^{R}]\rangle $
to represent an IVIFN, which satisfies
\begin{equation*}
\lbrack \mu _{\alpha }^{L},\mu _{\alpha }^{R}]\subset \lbrack 0,1],\ [\nu
_{\alpha }^{L},\nu _{\alpha }^{R}]\subset \lbrack 0,1],\text{ and }\mu
_{\alpha }^{R}+\nu _{\alpha }^{R}\leq 1,
\end{equation*}%
and use $\tilde{\Theta}$ to denote the set of all IVIFNs.

Xu \cite{X2007}\textrm{\ }introduced the score and accuracy functions for
IVIFNs and used them to compare two IVIFNs. However, there is no general
method which can rank any two arbitrary IVIFNs.

\begin{definition}[\textrm{\protect\cite{X2007}}]
Let $\alpha =\langle \lbrack \mu _{\alpha }^{L},\mu _{\alpha }^{R}],[\nu
_{\alpha }^{L},\nu _{\alpha }^{R}]\rangle $ be an IVIFN. A {\it score function} $S$
is defined as follows
\begin{equation*}
S(\alpha )=\frac{\mu _{\alpha }^{L}+\mu _{\alpha }^{R}}{2}-\frac{\nu
_{\alpha }^{L}+\nu _{\alpha }^{R}}{2},\text{ }S(\alpha )\in \left[ -1,1%
\right] .
\end{equation*}%
Moreover, the {\it accuracy function} is defined as follows%
\begin{equation*}
H(\alpha )=\frac{\mu _{\alpha }^{L}+\mu _{\alpha }^{R}}{2}+\frac{\nu
_{\alpha }^{L}+\nu _{\alpha }^{R}}{2}.
\end{equation*}
\end{definition}

\begin{definition}[\textrm{\protect\cite{X2007}}]
\label{def-xu} Let $\alpha _{1}$ and $\alpha _{2}$ be two IVIFNs. Then, the
ranking principle is defined as follows:
\begin{enumerate}[(1)]
\item If $S(\alpha _{1})<S(\alpha _{2})$, then $\alpha _{1}<\alpha _{2}$;

\item If $S(\alpha _{1})=S(\alpha _{2})$, then

\begin{itemize}
\item $H(\alpha _{1})<H(\alpha _{2})$, then $\alpha _{1}<\alpha _{2}$;

\item $H(\alpha _{1})=H(\alpha _{2})$, then $\alpha _{1}=\alpha _{2}$.
\end{itemize}
\end{enumerate}
\end{definition}

\section{Algebraic structures of $(\tilde{\Theta},\leq _{_{\mathrm{HZX}}})$}
\label{Sec-3}

In Definition {\ref{def-xu}}, a ranking procedure for IVIFNs by using a
score function and an accuracy function was proposed. Unfortunately, because
of the various properties of intervals, the score and accuracy functions
together may not determine the difference between two arbitrary IVIFNs. Hence, Huang
et al.~\cite{HZX2021} introduced a complete ranking method for IVIFNs via a
score function and three kinds of entropy functions, which are essential
notions for measuring the fuzziness of uncertain information. To be more
precise, they proposed the following order $\leq _{_{\mathrm{HZX}}}$ on
IVIFNs, which can rank any two arbitrary IVIFNs. They also proved that it is
a total order on IVIFNs.

\begin{definition}[{\textrm{\protect\cite[Definition 3.1]{HZX2021}}}]
Let $\alpha =\langle \lbrack \mu _{\alpha }^{L},\mu _{\alpha }^{R}],[\nu
_{\alpha }^{L},\nu _{\alpha }^{R}]\rangle $ be an IVIFN. Define the \textit{%
score function} $S$ of $\alpha $
\begin{equation*}
S(\alpha )=\frac{\mu _{\alpha }^{L}+\mu _{\alpha }^{R}}{2}-\frac{\nu
_{\alpha }^{L}+\nu _{\alpha }^{R}}{2}.
\end{equation*}
Moreover, three \textit{entropy functions} of $\alpha $ are defined as
follows:

\begin{itemize}
\item $\displaystyle E_{1}(\alpha )=\frac{1-\mu _{\alpha }^{L}-\mu _{\alpha
}^{R}}{2}+\frac{1-\nu _{\alpha }^{L}-\nu _{\alpha }^{R}}{2}$;

\item $\displaystyle E_{2}(\alpha )=\frac{\mu _{\alpha }^{R}-\mu _{\alpha
}^{L}+\nu _{\alpha }^{R}-\nu _{\alpha }^{L}}{2}$;

\item $\displaystyle E_{3}(\alpha )=\mu _{\alpha }^{R}-\mu _{\alpha }^{L}$.
\end{itemize}
\end{definition}

\begin{remark}
Notice that $E_{1}(\alpha )$ can be replaced by the accuracy function $%
H(\alpha )$ defined by%
\begin{equation*}
H(\alpha )=\frac{\mu _{\alpha }^{L}+\mu _{\alpha }^{R}}{2}+\frac{\nu
_{\alpha }^{L}+\nu _{\alpha }^{R}}{2}.
\end{equation*}
\end{remark}

\begin{definition}[{\textrm{\protect\cite[Definition 3.2]{HZX2021}}}]
\label{Def-order} Let $\alpha _{1}$ and $\alpha _{2}$ be two IVIFNs. Then,
it gets the following ranking principle:
\begin{enumerate}[(1)]
\item If $S(\alpha _{1})<S(\alpha _{2})$, then $\alpha _{1}$ is smaller than
$\alpha _{2}$, denoted by $\alpha _{1}<_{_{\mathrm{HZX}}}\alpha _{2}$;

\item If $S(\alpha _{1})=S(\alpha _{2})$, then

\begin{itemize}
\item $H(\alpha _{1})<H(\alpha _{2})$, then $\alpha _{1}<_{_{\mathrm{HZX}%
}}\alpha _{2}$;

\item $H(\alpha _{1})=H(\alpha _{2})$, then

\begin{itemize}
\item $E_{2}(\alpha _{1})<E_{2}(\alpha _{2})$, then $\alpha _{1}<_{_{\mathrm{%
HZX}}}\alpha _{2}$;

\item $E_{2}(\alpha _{1})=E_{2}(\alpha _{2})$, then
\begin{itemize}
\item $E_{3}(\alpha _{1})<E_{3}(\alpha _{2})$, then $\alpha _{1}<_{_{\mathrm{HZX}%
}}\alpha _{2}$;

\item $E_{3}(\alpha _{1})=E_{3}(\alpha _{2})$, then $\alpha _{1}=\alpha _{2}$.
\end{itemize}
\end{itemize}
\end{itemize}
\end{enumerate}
If $\alpha _{1}<_{_{\mathrm{HZX}}}\alpha _{2}$ or $\alpha _{1}=\alpha _{2}$,
we will denote it by $\alpha _{1}\leq _{_{\mathrm{HZX}}}\alpha _{2}$.
\end{definition}

\begin{theorem}[{\textrm{\protect\cite[Theorem 3.1]{HZX2021}}}]
\label{theo-chain}
Definition~\ref{Def-order} defines a total order on IVIFNs, i.e.,
the ranking principle given by Definition~\ref{Def-order}
gives a complete ranking in any class of IVIFNs.
\end{theorem}

\begin{remark}
\label{Remark-1} It is easy to see that the bottom and the top elements of $%
\tilde{\Theta}$ are, respectively, $\big\langle\lbrack 0,0],[1,1]\big\rangle$
and $\big\langle\lbrack 1,1],[0,0]\big\rangle$ with respect to the order $%
\leq _{_{\mathrm{HZX}}}$.
\end{remark}

\begin{definition}
\label{sub-order} Let $\alpha _{1}$ and $\alpha _{2}$ be two IVIFNs. A
relation $\subseteq $ on $\tilde{\Theta}$ is defined as: $\alpha \subseteq
\beta $ if and only if $\mu _{\alpha }^{L}\leq \mu _{\beta }^{L}$, $\mu
_{\alpha }^{R}\leq \mu _{\beta }^{R}$, $\nu _{\alpha }^{L}\geq \nu _{\beta
}^{L}$, and $\nu _{\alpha }^{R}\geq \nu _{\beta }^{R}$.
\end{definition}

\begin{definition}[{\textrm{\protect\cite[Definition~4.1]{DeBFIKM2016}}}]
A partial order $\preceq $ on $\tilde{\Theta}$ is said to be an \textit{%
admissible order} if it is a total order and refines the order $\subseteq $
introduced in Definition~\ref{sub-order}, i.e., it is a total order
satisfying that for any $\alpha $, $\beta \in \tilde{\Theta}$, $\alpha
\subseteq \beta $ implies $\alpha \preceq \beta $.
\end{definition}

In the following Theorem \ref{Completeness}, by applying a score function
and three types of entropy functions, we demonstrate that the space of all
IVIFSs with the order $\leq _{_{\mathrm{HZX}}}$ is a complete chain.

\begin{theorem}
\label{Completeness} $(\tilde{\Theta},\leq _{_{\mathrm{HZX}}})$ is a
complete chain.
\end{theorem}

\begin{proof}
By Theorem \ref{theo-chain}, $(\tilde{\Theta}, \leq_{_{\mathrm{HZX}}})$ is a chain.
Thus, it suffices to check that it is complete. Given a nonempty subset
$\Omega\subset \tilde{\Theta}$, we claim that the smallest upper bound of $\Omega$ exists.

For convenience, let $\mathscr{S}(\Omega)=\{S(\alpha)\mid \alpha\in \Omega\}$
and $\xi_{1}=\sup \mathscr{S}(\Omega)$. Then, we consider the following four cases:

\medskip

(1) $\xi_{1}\notin \mathscr{S}(\Omega)$ and $\xi_{1}\leq 0$. Take $\beta_1=\left\langle
  [0, 0], [-\xi_{1}, -\xi_{1}]\right\rangle$.
  Clearly, $\beta_1\in \tilde{\Theta}$. Meanwhile, for any $\alpha\in \Omega$, it is clear that
  $S(\alpha)< \sup\mathscr{S}(\Omega)=\xi_{1}$ by $\xi_{1}\notin \mathscr{S}(\Omega)$. This, together with
  Definition~\ref{Def-order}, implies that $\alpha<_{_{\mathrm{HZX}}}\beta_1$, i.e., $\beta_1$ is
  an upper bound of $\Omega$. Given an upper bound $\beta=\langle [\mu_{\beta}^{L}, \mu_{\beta}^{R}],
  [\nu_{\beta}^{L}, \nu_{\beta}^{R}]\rangle\in \tilde{\Theta}$ of $\Omega$,
  by Definition~\ref{Def-order}, we have that $S(\beta)\geq S(\alpha)$ holds for all $\alpha\in \Omega$,
  implying that $S(\beta)\geq \sup\{S(\alpha) \mid \alpha\in \Omega\}=\xi_{1}$.
  \begin{itemize}

  \item If $S(\beta)>\xi_{1}=S(\beta_1)$, by Definition~\ref{Def-order}, it is clear that $\beta>_{_{\mathrm{HZX}}}\beta_1$.

  \item If $S(\beta)=\xi_{1}=S(\beta_1)$, by $S(\beta)=\frac{\mu_{\beta}^{L}+\mu_{\beta}^{R}}{2}
    -\frac{\nu_{\beta}^{L}+\nu_{\beta}^{R}}{2}\geq -\frac{\nu_{\beta}^{L}+\nu_{\beta}^{R}}{2}$, we have $H(\beta)=\frac{\mu_{\beta}^{L}+\mu_{\beta}^{R}}{2}
    +\frac{\nu_{\beta}^{L}+\nu_{\beta}^{R}}{2}\geq \frac{\nu_{\beta}^{L}+\nu_{\beta}^{R}}{2}
    \geq -S(\beta)=-\xi_{1}$.
  \begin{itemize}
    \item If $H(\beta)>-\xi_{1}$, by Definition~\ref{Def-order} and $H(\beta_1)=-\xi_{1}$,
      we have $\beta>_{_{\mathrm{HZX}}} \beta_1$.

    \item If $H(\beta)=-\xi_{1}$, noting that $H(\beta_1)=-\xi_{1}$, $E_{2}(\beta_1)=E_{3}(\beta_1)=0$,
      and $E_{2}(\beta)\geq 0$, $E_{3}(\beta)\geq 0$, by Definition~\ref{Def-order}, we have
      $\beta\geq _{_{\mathrm{HZX}}}\beta_1$.
  \end{itemize}
  \end{itemize}

  Therefore, $\beta_1$ is the smallest upper bound of $\Omega$.

  \medskip

  (2) $\xi_{1}\notin \mathscr{S}(\Omega)$ and $\xi_{1}\geq 0$.
  Similarly to the proof of (1), it can be verified
  that $\beta_2=\langle[\xi_{1}, \xi_{1}], [0, 0]\rangle$
  is the smallest upper bound of $\Omega$.

  \medskip

  (3) $\xi_{1}\in \mathscr{S}(\Omega)$ and $\xi_{1}\leq 0$. This implies that
  $\bar{\Omega}=\{\alpha\in \Omega\mid S(\alpha)=\xi_{1}\}\neq \varnothing$. Then,
  let us take $\xi_{2}=\sup \{H(\alpha) \mid \alpha\in \bar{\Omega}\}$.

  \medskip

  3.1) If $\xi_{2}\notin \{H(\alpha) \mid \alpha\in \bar{\Omega}\}$, then, for any
  $n\in \mathbb{N}$, there exists $\alpha_n\in \bar{\Omega}$ such that
  $\xi_2-\frac{1}{n}<H(\alpha_n)<\xi_2$,  %%three exists
    %%$\alpha\in \bar{\Omega}$ such that $H(\alpha)=\xi_{2}$,
  i.e.,
    $$
    \begin{cases}
    S(\alpha_n)=\frac{\mu_{\alpha_n}^{L}+\mu_{\alpha_n}^{R}}{2}
    -\frac{\nu_{\alpha_n}^{L}+\nu_{\alpha_n}^{R}}{2}=\xi_1,\\
    \xi_2-\frac{1}{n}<H(\alpha_n)=\frac{\mu_{\alpha_n}^{L}+\mu_{\alpha_n}^{R}}{2}
    +\frac{\nu_{\alpha_n}^{L}+\nu_{\alpha_n}^{R}}{2}<\xi_2.
    \end{cases}
    $$
    This, together with $\alpha_n\in \tilde{\Theta}$, implies that
    $$
    \xi_2=\lim_{n\to +\infty}\left(\frac{\mu_{\alpha_n}^{L}+\mu_{\alpha_n}^{R}}{2}
    +\frac{\nu_{\alpha_n}^{L}+\nu_{\alpha_n}^{R}}{2}\right)
    \leq \lim_{n\to +\infty}\left(\mu_{\alpha_n}^{R}+\nu_{\alpha_n}^{R}\right)\leq 1,
    $$
    $$
    \frac{\xi_{1}+\xi_{2}}{2}=\lim_{n\to +\infty}\frac{S(\alpha_n)+H(\alpha_n)}{2}
    =\lim_{n\to +\infty}\frac{\mu_{\alpha_n}^{L}+\mu_{\alpha_n}^{R}}{2}\in [0, 1],
    $$
    and
    $$
    \frac{\xi_{2}-\xi_{1}}{2}=\lim_{n\to +\infty}\frac{H(\alpha_n)-S(\alpha_n)}{2}
    =\lim_{n\to +\infty}\frac{\nu_{\alpha_n}^{L}+\nu_{\alpha_n}^{R}}{2}\in [0, 1],
    $$
    and thus $\beta_2=\left\langle\left[\frac{\xi_{1}+\xi_{2}}{2},
    \frac{\xi_{1}+\xi_{2}}{2}\right], \left[\frac{\xi_{2}-\xi_{1}}{2},
    \frac{\xi_{2}-\xi_{1}}{2}\right] \right\rangle \in \tilde{\Theta}$.
    By direct calculation, we have
    $S(\beta_2)=\xi_{1}$ and $H(\beta_2)=\xi_{2}$. For any $\alpha\in \Omega$,
    from the choice of $\xi_{1}$,
    it follows that $S(\alpha)\leq \xi_{1}=S(\beta_2)$.
    \begin{itemize}
      \item If $S(\alpha)<S(\beta_2)$, by Definition~\ref{Def-order}, it is clear that $\alpha
      <_{_{\mathrm{HZX}}}\beta_2$.

      \item If $S(\alpha)=\xi_{1}$, i.e., $\alpha\in \bar{\Omega}$, by the choice of $\xi_{2}$ and
      $\xi_{2}\notin \{H(\alpha) \mid \alpha\in \bar{\Omega}\}$, we have
      $H(\alpha)<\xi_{2}=H(\beta_2)$, and thus $\alpha<_{_{\mathrm{HZX}}} \beta_2$
      by Definition~\ref{Def-order}.
    \end{itemize}
    These imply that $\beta_2$ is an upper bound of $\Omega$.
    Given an upper bound
    $\beta=\langle [\mu_{\beta}^{L}, \mu_{\beta}^{R}], [\nu_{\beta}^{L}, \nu_{\beta}^{R}]\rangle
    \in \tilde{\Theta}$ of $\Omega$, by Definition~\ref{Def-order},
    it is clear that $S(\beta)\geq \xi_{1}$.
    \begin{itemize}
      \item If $S(\beta)>\xi_{1}$, by Definition~\ref{Def-order} and $S(\beta_2)=\xi_{1}$,
      it is clear that $\beta>_{_{\mathrm{HZX}}} \beta_2$.

      \item If $S(\beta)=\xi_{1}$, for any $\alpha\in \bar{\Omega}$,
      by $\beta\geq_{_{\mathrm{HZX}}} \alpha$ and $S(\beta)=S(\alpha)$,
      then $H(\beta)\geq H(\alpha)$, and thus $H(\beta)
      \geq \sup\{H(\alpha) \mid \alpha\in \bar{\Omega}\}=\xi_{2}=H(\beta_2)$.
      This, together with $E_2(\beta_2)
      =0\leq E_{2}(\beta)$ and $E_{3}(\beta_2)=0\leq E_{3}(\beta)$, implies that
      $\beta \geq_{_{\mathrm{HZX}}} \beta_2$.
    \end{itemize}
    Therefore, $\beta_2$ is the smallest upper bound of $\Omega$.

    \medskip

    3.2) If $\xi_{2}\in \{H(\alpha) \mid \alpha\in \bar{\Omega}\}$, i.e.,
    $\bar{\Omega}_1=\{\alpha\in \bar{\Omega} \mid H(\alpha)=\xi_{2}\}\neq \varnothing$,
   then let us take $\xi_{3}=\sup\{E_{2}(\alpha) \mid \alpha\in \bar{\Omega}_1\}$.

   \medskip

   3.2.1) If $\xi_{3}\in \{E_{2}(\alpha) \mid \alpha\in \bar{\Omega}_1\}$, i.e.,
     $\{\alpha\in \bar{\Omega}_1 \mid E_{2}(\alpha)=\xi_{3}\}\neq \varnothing$,
     then we take
     $\bar{\Omega}_2=\{\alpha\in \bar{\Omega}_1 \mid E_{2}(\alpha)=\xi_{3}\}$ and
     $\xi_{4}=\sup\{E_{3}(\alpha) \mid \alpha\in \bar{\Omega}_2\}$, and consider
     the following two subcases:

     \medskip

     i) If $\xi_{4}\notin \{E_{3}(\alpha) \mid \alpha\in \bar{\Omega}_2\}$, then we choose
     $$
     \beta_3=\left\langle\left[\frac{\xi_{1}+\xi_{2}-\xi_{4}}{2},
     \frac{\xi_{1}+\xi_{2}+\xi_{4}}{2}\right], \left[\frac{\xi_{2}-\xi_{1}}{2}-\frac{2\xi_{3}-\xi_{4}}{2},
     \frac{\xi_{2}-\xi_{1}}{2}+\frac{2\xi_{3}-\xi_{4}}{2}\right]\right\rangle.
     $$
     \begin{claim}\label{Claim-1}
     $\beta_3\in \tilde{\Theta}$.
     \end{claim}

     {\bf Proof of Claim~\ref{Claim-1}:}
     By $\xi_{4}=\sup\{E_{3}(\alpha) \mid \alpha\in \bar{\Omega}_2\}
     \notin \{E_{3}(\alpha) \mid \alpha\in \bar{\Omega}_2\}$,
     we have that, for any $n\in \mathbb{N}$, there exists
     $\alpha_{n}=\langle[\mu_{\alpha_n}^{L}, \mu_{\alpha_n}^{R}],
     [\nu_{\alpha_n}^{L}, \nu_{\alpha_n}^{R}]\rangle\in \bar{\Omega}_2$
     (i.e., $\alpha_{n}\in \bar{\Omega}_1$
     and $E_{2}(\alpha_n)=\xi_{3}$) such that
     \begin{equation}
     \label{eq-Wu-sup}
     \xi_{4}-\frac{1}{n}<E_{3}(\alpha_n)
     =\mu_{\alpha_n}^{R}-\mu_{\alpha_n}^{L}<
     \xi_{4}, \text{ i.e., } \lim_{n\to +\infty}E_{3}(\alpha_n)=\xi_{4}.
     \end{equation}
     From $\langle[\mu_{\alpha_n}^{L}, \mu_{\alpha_n}^{R}],
     [\nu_{\alpha_n}^{L}, \nu_{\alpha_n}^{R}]\rangle\in \bar{\Omega}_1$, it follows that
     $$
   \begin{cases}
   \frac{\mu_{\alpha_n}^{L}+\mu_{\alpha_n}^{R}}{2}-\frac{\nu_{\alpha_n}^{L}+\nu_{\alpha_n}^{R}}{2}=\xi_{1},\\
   \frac{\mu_{\alpha_n}^{L}+\mu_{\alpha_n}^{R}}{2}+\frac{\nu_{\alpha_n}^{L}+\nu_{\alpha_n}^{R}}{2}=\xi_{2},
     \end{cases}
     $$
     and thus
     $$
     \frac{\mu_{\alpha_n}^{L}+\mu_{\alpha_n}^{R}}{2}=\frac{\xi_{1}+\xi_{2}}{2}
     \text{ and } \frac{\nu_{\alpha_n}^{L}+\nu_{\alpha_n}^{R}}{2}=\frac{\xi_{2}-\xi_{1}}{2}.
     $$
     This, together with $E_{2}(\alpha_n)=
     \frac{\mu_{\alpha_n}^{R}-\mu_{\alpha_n}^{L}
     +\nu_{\alpha_n}^{R}-\nu_{\alpha_n}^{L}}{2}=\xi_{3}$ and
     $E_{3}(\alpha_n)=\mu_{\alpha_n}^{R}-\mu_{\alpha_n}^{L}$,
     implies that
     $$
     [\mu_{\alpha_n}^{L}, \mu_{\alpha_n}^{R}]
     =\left[\frac{\xi_{1}+\xi_{2}}{2}-\frac{E_{3}(\alpha_n)}{2},
     \frac{\xi_{1}+\xi_{2}}{2}+\frac{E_{3}(\alpha_n)}{2}\right]\subset [0, 1],
     $$
     and
     $$
     [\nu_{\alpha_n}^{L}, \nu_{\alpha_n}^{R}]
     =\left[\frac{\xi_{2}-\xi_{1}}{2}-\frac{2\xi_{3}-E_{3}(\alpha_n)}{2},
     \frac{\xi_{2}-\xi_{1}}{2}+\frac{2\xi_{3}-E_{3}(\alpha_n)}{2}\right] \subset [0, 1].
     $$
     By $\alpha_n=\langle[\mu_{\alpha_n}^{L}, \mu_{\alpha_n}^{R}],
     [\nu_{\alpha_n}^{L}, \nu_{\alpha_n}^{R}]\rangle\in \tilde{\Theta}$, we have
     $$
     \frac{\xi_{1}+\xi_{2}}{2}-\frac{E_{3}(\alpha_n)}{2}\geq 0,
     $$
     $$
     \frac{\xi_{2}-\xi_{1}}{2}-\frac{2\xi_{3}-E_{3}(\alpha_n)}{2}\geq 0,
     $$
     and
     $$
     \left[\frac{\xi_{1}+\xi_{2}}{2}+\frac{E_{3}(\alpha_n)}{2}\right]+
     \left[\frac{\xi_{2}-\xi_{1}}{2}+\frac{2\xi_{3}-E_{3}(\alpha_n)}{2}\right]
     =\xi_{2}+\xi_{3} \leq 1.
     $$
     These, together with formula~\eqref{eq-Wu-sup}, imply that
     \begin{equation}
     \label{eq-Wu-2}
     \frac{\xi_{1}+\xi_{2}}{2}-\frac{\xi_{4}}{2}
     =\lim_{n\to +\infty}\left(\frac{\xi_{1}+\xi_{2}}{2}
     -\frac{E_{3}(\alpha_n)}{2}\right)
     \geq 0,
     \end{equation}
     \begin{equation}
     \label{eq-Wu-3}
     \frac{\xi_{2}-\xi_{1}}{2}-\frac{2\xi_{3}-\xi_{4}}{2}=
     \lim_{n\to +\infty}\left(\frac{\xi_{2}-\xi_{1}}{2}-
     \frac{2\xi_{3}-E_{3}(\alpha_n)}{2}\right)\geq 0,
     \end{equation}
     and
     \begin{equation}
     \label{eq-Wu-4}
     \frac{\xi_{1}+\xi_{2}+\xi_{4}}{2}+\left(
     \frac{\xi_{2}-\xi_{1}}{2}+\frac{2\xi_{3}-\xi_{4}}{2}\right)
     =\xi_{2}+\xi_{3} \leq 1,
     \end{equation}
     i.e., $\beta_3\in \tilde{\Theta}$.

     \medskip

     By direct calculation, we have
     \begin{claim}
     \label{Claim-2}
     $S(\beta_3)=\xi_{1}$, $H(\beta_3)=\xi_{2}$,
     $E_{2}(\beta_3)=\xi_{3}$, and $E_{3}(\beta_3)=\xi_{4}$.
     \end{claim}

     \begin{claim}
     \label{Claim-3}
     $\beta_3$ is an upper bound of $\Omega$.
     \end{claim}

     {\bf Proof of Claim~\ref{Claim-3}:}
     For any $\alpha\in \Omega$,
     \begin{itemize}
       \item If $\alpha\in \Omega\setminus \bar{\Omega}$, by the choice of $\bar{\Omega}$,
       we have $S(\alpha)<\xi_{1}=S(\beta_3)$ (by Claim~\ref{Claim-2}), and thus
       $\alpha<_{_{\mathrm{HZX}}} \beta_3$ by Definition~\ref{Def-order}.

       \item If $\alpha\in \bar{\Omega}\setminus \bar{\Omega}_1$, by the choices of $\bar{\Omega}$
       and $\bar{\Omega}_1$, we have $S(\alpha)=\xi_{1}=S(\beta_3)$ and $H(\alpha)<\xi_{2}=H(\beta_3)$
        (by Claim~\ref{Claim-2}), and thus $\alpha<_{_{\mathrm{HZX}}} \beta_3$ by Definition~\ref{Def-order}.

       \item If $\alpha\in \bar{\Omega}_1\setminus \bar{\Omega}_2$, by the choices of $\bar{\Omega}_1$
       and $\bar{\Omega}_2$, we have $S(\alpha)=\xi_{1}=S(\beta_3)$, $H(\alpha)=\xi_{2}=H(\beta_3)$, and $E_{2}(\alpha)<\xi_{3}=E_{2}(\beta_3)$ (by Claim~\ref{Claim-2}), and thus $\alpha<_{_{\mathrm{HZX}}} \beta_3$ by Definition~\ref{Def-order}.

       \item If $\alpha\in \bar{\Omega}_2$, by the choices of $\bar{\Omega}_2$ and $\xi_{4}$, we have $S(\alpha)=\xi_{1}=S(\beta_3)$, $H(\alpha)=\xi_{2}=H(\beta_3)$,  $E_{2}(\alpha)=\xi_{3}=E_{2}(\beta_3)$, and $E_{3}(\alpha)<\xi_{4}=E_{3}(\beta_3)$ (by Claim~\ref{Claim-2}). and thus $\alpha<_{_{\mathrm{HZX}}} \beta_3$ by Definition~\ref{Def-order}.
     \end{itemize}

     \begin{claim}
     \label{Claim-4}
     $\beta_3$ is the smallest upper bound of $\Omega$.
     \end{claim}

     {\bf Proof of Claim~\ref{Claim-4}:} Given an upper bound
    $\beta=\langle[\mu_{\beta}^{L}, \mu_{\beta}^{R}], [\nu_{\beta}^{L}, \nu_{\beta}^{R}]\rangle
    \in \tilde{\Theta}$ of $\Omega$, by Definition~\ref{Def-order},
    it is clear that $S(\beta)\geq \xi_{1}$.
    \begin{itemize}
      \item If $S(\beta)>\xi_{1}$, by Definition~\ref{Def-order} and $S(\beta_3)=\xi_{1}$,
      it is clear that $\beta>_{_{\mathrm{HZX}}} \beta_3$.
      \item If $S(\beta)=\xi_{1}$, for $\alpha\in \bar{\Omega}$,
       by $\beta\geq_{_{\mathrm{HZX}}} \alpha$, then $H(\beta)\geq H(\alpha)$,
       and thus $H(\beta) \geq \sup\{H(\alpha) \mid \alpha\in \bar{\Omega}\}=\xi_{2}=H(\beta_3)$.
      \medskip
      \begin{itemize}
      \item If $H(\beta)> H(\beta_3)$, it is clear that $\beta >_{_{\mathrm{HZX}}} \beta_3$
      by Definition~\ref{Def-order}.

      \medskip

      \item If $H(\beta)=H(\beta_3)$, by $\beta\geq_{_{\mathrm{HZX}}} \alpha$
      for all $\alpha\in \bar{\Omega}_1$, we have $E_{2}(\beta)\geq \sup\{E_{2}(\alpha) \mid
      \alpha\in \bar{\Omega}_1\}=\xi_{3}=E_{2}(\beta_3)$.

      \begin{itemize}

       \item If $E_{2}(\beta)>E_{2}(\beta_3)$, it is clear that $\beta >_{_{\mathrm{HZX}}} \beta_3$
      by Definition~\ref{Def-order}.

       \medskip

       \item If $E_{2}(\beta)= E_{2}(\beta_3)$, by $\beta\geq_{_{\mathrm{HZX}}} \alpha$
      for all $\alpha\in \bar{\Omega}_2$, we have $E_{3}(\beta)\geq \sup\{E_{3}(\alpha) \mid
      \alpha\in \bar{\Omega}_2\}=\xi_{4}=E_{3}(\beta_3)$, and thus
      $\beta \geq_{_{\mathrm{HZX}}} \beta_3$ by Definition~\ref{Def-order}.
      \end{itemize}
      \end{itemize}

    Therefore, $\beta_3$ is the smallest upper bound of $\Omega$.

    \medskip

       ii) If $\xi_{4}\in \{E_{3}(\alpha) \mid \alpha\in \bar{\Omega}_2\}$, i.e., there exists
       $\alpha\in \bar{\Omega}_2$ such that $E_{3}(\alpha)=\xi_4$, then
       $$
       \begin{cases}
       S(\alpha)=\frac{\mu_{\alpha}^{L}+\mu_{\alpha}^{R}}{2}-\frac{\nu_{\alpha}^{L}+\nu_{\alpha}^{R}}{2}=\xi_1,\\
       H(\alpha)=\frac{\mu_{\alpha}^{L}+\mu_{\alpha}^{R}}{2}+\frac{\nu_{\alpha}^{L}+\nu_{\alpha}^{R}}{2}=\xi_2,\\
       E_{2}(\alpha)=\frac{\mu_{\alpha}^{R}-\mu_{\alpha}^{L}+\nu_{\alpha}^{R}-\nu_{\alpha}^{L}}{2}=\xi_3,\\
       E_{3}(\alpha)=\mu_{\alpha}^{R}-\mu_{\alpha}^{L}=\xi_4,
       \end{cases}
       $$
       i.e.,
       $$
       \begin{bmatrix}
       \frac{1}{2} & \frac{1}{2} & -\frac{1}{2} & -\frac{1}{2}\\
       \frac{1}{2} & \frac{1}{2} & \frac{1}{2} & \frac{1}{2}\\
       -\frac{1}{2} & \frac{1}{2} & -\frac{1}{2} & \frac{1}{2}\\
       -1 & 1 & 0 & 0
       \end{bmatrix}
       \begin{bmatrix}
       \mu_{\alpha}^{L} \\
       \mu_{\alpha}^{R} \\
       \nu_{\alpha}^{L} \\
       \nu_{\alpha}^{R}
       \end{bmatrix}
       =
       \begin{bmatrix}
       \xi_1 \\
       \xi_2 \\
       \xi_3 \\
       \xi_4
       \end{bmatrix}.
       $$
       Since $
       \begin{vmatrix}
       \frac{1}{2} & \frac{1}{2} & -\frac{1}{2} & -\frac{1}{2}\\
       \frac{1}{2} & \frac{1}{2} & \frac{1}{2} & \frac{1}{2}\\
       -\frac{1}{2} & \frac{1}{2} & -\frac{1}{2} & \frac{1}{2}\\
       -1 & 1 & 0 & 0
       \end{vmatrix}=1\neq 0$, we obtain that such $\alpha$
       ($\alpha\in \bar{\Omega}_2$ and $E_{3}(\alpha)=\xi_4$)
       is unique, which is denoted by $\hat{\alpha}$.
       Under this condition, it can be verified that $\hat{\alpha}$ is the
        maximum of $\Omega$, and thus it is the smallest upper bound of $\Omega$.
     \end{itemize}

     3.2.2) If $\xi_{3}\notin \{E_{2}(\alpha) \mid \alpha\in \bar{\Omega}_1\}$, by
     $\bar{\Omega}_1\neq \varnothing$, there exists $\alpha\in \bar{\Omega}$ such that
     $H(\alpha)=\xi_2$, i.e.,
     $$
     \begin{cases}
     S(\alpha)=\frac{\mu_{\alpha}^{L}+\mu_{\alpha}^{R}}{2}-\frac{\nu_{\alpha}^{L}+\nu_{\alpha}^{R}}{2}
     =\xi_1,\\
     H(\alpha)=\frac{\mu_{\alpha}^{L}+\mu_{\alpha}^{R}}{2}+\frac{\nu_{\alpha}^{L}+\nu_{\alpha}^{R}}{2}
     =\xi_2,
     \end{cases}
     $$
     implying that
     $$
     \zeta_1:=\frac{\xi_1+\xi_2}{2}=\frac{\mu_{\alpha}^{L}+\mu_{\alpha}^{R}}{2}\in [0, 1]
     \text{ and }
     \zeta_2:=\frac{\xi_2-\xi_1}{2}=\frac{\nu_{\alpha}^{L}+\nu_{\alpha}^{R}}{2}\in [0, 1].
     $$
     Clearly, $\zeta_1\leq \zeta_2$ and $\zeta_1+\zeta_2
     =\frac{\mu_{\alpha}^{L}+\mu_{\alpha}^{R}}{2}+\frac{\nu_{\alpha}^{L}+\nu_{\alpha}^{R}}{2}
     \leq \mu_{\alpha}^{R}+\nu_{\alpha}^{R}\leq 1$.

      \begin{claim}
      \label{Claim-5}
      $\zeta_1, \zeta_2, \xi_3\geq 0$, $\zeta_1+\zeta_2+\xi_3\leq 1$,
      and $\zeta_1+\zeta_2-\xi_3\geq 0$.
      \end{claim}

     {\bf Proof of Claim~\ref{Claim-5}:} Clearly, $\zeta_1, \zeta_2, \xi_3\geq 0$.
     By $\xi_{3}=\sup\{E_{2}(\alpha) \mid \alpha\in \bar{\Omega}_1\}
     \notin \{E_{2}(\alpha) \mid \alpha\in \bar{\Omega}_1\}$,
     we have that, for any $n\in \mathbb{N}$, there exists $\alpha_n\in \bar{\Omega}_1$ such that
     $\xi_3-\frac{1}{n}< E_{2}(\alpha_n)<\xi_3$, implying that
       $$
       \frac{(\mu_{\alpha_n}^{L}+\mu_{\alpha_n}^{R})
       +(\nu_{\alpha_n}^{L}+\nu_{\alpha_n}^{R})}{2}=\xi_2=\zeta_1+\zeta_2
       \text{ (by $\alpha_n\in \bar{\Omega}_1$)},
       $$
       and
       $$
       \xi_3-\frac{1}{n}<E_{2}(\alpha_n)=\frac{(\mu_{\alpha_n}^{R}-\mu_{\alpha_n}^{L})
       +(\nu_{\alpha_n}^{R}-\nu_{\alpha_n}^{L})}{2}<\xi_3,
       $$
       and thus
       \begin{align*}
       \xi_3-\frac{1}{n}+(\zeta_1+\zeta_2)
       &<\left[\frac{(\mu_{\alpha_n}^{L}+\mu_{\alpha_n}^{R})
       +(\nu_{\alpha_n}^{L}+\nu_{\alpha_n}^{R})}{2}
       +\frac{(\mu_{\alpha_n}^{R}-\mu_{\alpha_n}^{L})
       +(\nu_{\alpha_n}^{R}-\nu_{\alpha_n}^{L})}{2}\right]\\
       &=\mu_{\alpha_n}^{R}+\nu_{\alpha_n}^{R}<\xi_3+(\zeta_1+\zeta_2),
       \end{align*}
       and
       \begin{align*}
       0\leq \mu_{\alpha_n}^{L}+\nu_{\alpha_n}^{L}&=\left[\frac{(\mu_{\alpha_n}^{L}+\mu_{\alpha_n}^{R})
       +(\nu_{\alpha_n}^{L}+\nu_{\alpha_n}^{R})}{2}-\frac{(\mu_{\alpha_n}^{R}-\mu_{\alpha_n}^{L})
       +(\nu_{\alpha_n}^{R}-\nu_{\alpha_n}^{L})}{2}\right]\\
       &<(\zeta_1+\zeta_2)-\left(\xi_3-\frac{1}{n}\right).
       \end{align*}
       By $\mu_{\alpha_n}^{R}+\nu_{\alpha_n}^{R}\leq 1$
       ($n\in \mathbb{N}$), letting $n\to +\infty$, we have
       $$
       1\geq \lim_{n\to +\infty}(\mu_{\alpha_n}^{R}+\nu_{\alpha_n}^{R})=\xi_3+(\zeta_1+\zeta_2),
       $$
       and
       $$
       \zeta_1+\zeta_2-\xi_3\geq 0.
       $$

Let us choose
$$
\beta_4=
\begin{cases}
\big\langle [\zeta_1, \zeta_1], [\zeta_2-\xi_3, \zeta_2+\xi_3]\big\rangle, & \zeta_2-\xi_3\geq 0,\\
\big\langle [\zeta_1-(\xi_3-\zeta_2), \zeta_1+(\xi_3-\zeta_2)], [0, 2\zeta_2]\big\rangle, & \zeta_2-\xi_3< 0.
\end{cases}
$$

      \begin{claim}
      \label{Claim-6}
      $\beta_4\in \tilde{\Theta}.$
      \end{claim}

     {\bf Proof of Claim~\ref{Claim-6}:} If $\zeta_2-\xi_3\geq 0$, by Claim~\ref{Claim-5},
      we have $[\zeta_1, \zeta_1] \subset
     [0, 1]$, $[\zeta_2-\xi_3, \zeta_2+\xi_3]\subset [0, 1]$, and
     $\zeta_1+(\zeta_2+\xi_3)\leq 1$, and thus $\beta_4\in \tilde{\Theta}$.

     If $\zeta_2-\xi_3< 0$, by Claim~\ref{Claim-5},
      we have $0\leq \zeta_1-(\xi_3-\zeta_2)\leq
      \zeta_1+(\xi_3-\zeta_2)\leq \zeta_1+(\xi_3+\zeta_2)\leq 1$,
      $2\zeta_2\leq \zeta_2+\xi_3\leq \zeta_1+\zeta_2+\xi_3\leq 1$,
      i.e., $[\zeta_1-(\xi_3-\zeta_2), \zeta_1+(\xi_3-\zeta_2)]
      \subset [0, 1]$, $[0, 2\zeta_2]\subset [0, 1]$; and
      $\zeta_1+(\xi_3-\zeta_2)+2\zeta_2=\zeta_1+\zeta_2+\xi_3\leq 1$, and thus
      $\beta_4\in \tilde{\Theta}$.

      \medskip

      By direct calculation, we have
      $S(\beta_4)=\xi_1$, $H(\beta_4)=\xi_2$, and $E_{2}(\beta_4)=\xi_3$.
      Similarly to the proof of Claim~\ref{Claim-3}, it can be verified that

     \begin{claim}
     $\beta_4$ is an upper bound of $\Omega$.
     \end{claim}

     \begin{claim}
     \label{Claim-7}
     $\beta_4$ is the smallest upper bound of $\Omega$.
     \end{claim}

     {\bf Proof of Claim~\ref{Claim-7}:} Given an upper bound
    $\beta=\langle[\mu_{\beta}^{L}, \mu_{\beta}^{R}], [\nu_{\beta}^{L}, \nu_{\beta}^{R}]\rangle
    \in \tilde{\Theta}$ of $\Omega$, by Definition~\ref{Def-order},
    it is clear that $S(\beta)\geq \xi_{1}$.
    \begin{itemize}
      \item If $S(\beta)>\xi_{1}$, by Definition~\ref{Def-order} and $S(\beta_4)=\xi_{1}$,
      it is clear that $\beta>_{_{\mathrm{HZX}}} \beta_4$.
      \item If $S(\beta)=\xi_{1}$, for any $\alpha\in \bar{\Omega}$, by $\beta\geq_{_{\mathrm{HZX}}} \alpha$,
      then $H(\beta)\geq H(\alpha)$, and thus $H(\beta) \geq \sup\{H(\alpha) \mid \alpha\in \bar{\Omega}\}=\xi_{2}=H(\beta_4)$.

      \begin{itemize}
      \item If $H(\beta)> H(\beta_4)$, it is clear that $\beta >_{_{\mathrm{HZX}}} \beta_4$
      by Definition~\ref{Def-order}.

      \item If $H(\beta)=H(\beta_4)$, by $\beta\geq_{_{\mathrm{HZX}}} \alpha$
      for all $\alpha\in \bar{\Omega}_1$, we have $E_{2}(\beta)\geq \sup\{E_{2}(\alpha) \mid
      \alpha\in \bar{\Omega}_1\}=\xi_{3}=E_{2}(\beta_4)$.
      \begin{itemize}

       \item If $E_{2}(\beta)>E_{2}(\beta_4)$, it is clear that $\beta >_{_{\mathrm{HZX}}} \beta_4$
      by Definition~\ref{Def-order}.

       \item If $E_{2}(\beta)= E_{2}(\beta_4)$, by $S(\beta)=\xi_1$ and $H(\beta)=\xi_2$,
       then $\frac{\mu_{\beta}^{L}+\mu_{\beta}^{R}}{2}
       =\zeta_1$, $\frac{\nu_{\beta}^{L}+\nu_{\beta}^{R}}{2}=\zeta_2$, and $\frac{\mu_{\beta}^{R}-\mu_{\beta}^{L}}{2}+\frac{\nu_{\beta}^{R}-\nu_{\beta}^{L}}{2}=\xi_3$.
       Therefore, we can assume that $[\mu_{\beta}^{L}, \mu_{\beta}^{R}]=
       [\zeta_1-l_1, \zeta_1+l_1]\subset [0, 1]$ and $[\nu_{\beta}^{L}, \nu_{\beta}^{R}]=
       [\zeta_2-l_2, \zeta_1+l_2]\subset [0, 1]$, where $l_1$, $l_2\geq 0$ and $l_1+l_2=\xi_3$.
       Consider the following two subcases:
       \begin{itemize}
          \item If $\zeta_2-\xi_3\geq 0$, by the choice of $\beta_4$, we have
          $E_{3}(\beta_4)=0\leq E_{3}(\beta)$. This, together with $S(\beta_4)
          =S(\beta)$, $H(\beta_4)= H(\beta)$, and $E_{2}(\beta_4)= E_{2}(\beta)$,
          implies that $\beta_4 \leq_{_{\mathrm{HZX}}} \beta$.

          \item If $\zeta_2-\xi_3< 0$, by $\zeta_2-l_2\geq 0$, we have $l_2\leq \zeta_2 <\xi_3$,
          implying that $\xi_3-\zeta_2<\zeta_3-l_2=l_1$. This, together with $\beta_4=
              \langle [\zeta_1-(\xi_3-\zeta_2), \zeta_1+(\xi_3-\zeta_2)], [0, 2\zeta_2]\rangle$,
              implies that $E_{3}(\beta_4)=2(\xi_3-\zeta_2)<2l_1=E_{3}(\beta)$, and thus
              $\beta_4 <_{_{\mathrm{HZX}}} \beta$ since $S(\beta_4)
          =S(\beta)$, $H(\beta_4)= H(\beta)$, and $E_{2}(\beta_4)= E_{2}(\beta)$.
        \end{itemize}
       \end{itemize}
    \end{itemize}
    \end{itemize}

    Therefore, $\beta_4$ is the smallest upper bound of $\Omega$.

    \medskip

  (4) $\xi_{1}\in \mathscr{S}(\Omega)$ and $\xi_{1}\geq 0$.
  Similarly to the proof of (3), it can be verified
  that the smallest upper bound of $\Omega$ exists.

Summing up above, we get that the smallest upper bound of $\Omega$ exists. Hence,
$(\tilde{\Theta}, \leq_{_{\mathrm{HZX}}})$ is a complete lattice by applying
Lemma~\ref{Complete-Char} and Remark~\ref{Remark-1}.
\end{proof}

In the following Theorem \ref{theo-admissible}, we show that the order $\leq
_{_{\mathrm{HZX}}}$ fulfills the order $\subseteq $ introduced in Definition~%
\ref{sub-order}; namely, it is an admissible order on IVIFNs.

\begin{theorem}
\label{theo-admissible}
The order $\leq _{_{\mathrm{HZX}}}$ in Definition~\ref{Def-order} is an admissible order on $%
\tilde{\Theta}$.
\end{theorem}

\begin{proof}
Based on Theorem {\ref{theo-chain}, }$\leq _{_{\mathrm{HZX}}}$ is a total
order{\ }on $\tilde{\Theta}.$ For two IVIFNs $\alpha =\langle \lbrack \mu
_{\alpha }^{L},\mu _{\alpha }^{R}],[\nu _{\alpha }^{L},\nu _{\alpha
}^{R}]\rangle $ and $\beta =\langle \lbrack \mu _{\beta }^{L},\mu _{\beta
}^{R}],[\nu _{\beta }^{L},\nu _{\beta }^{R}]\rangle ,$ let $\alpha \subseteq
\beta .$ By Definition {\ref{sub-order}, there holds} $\mu _{\alpha }^{L}\leq
\mu _{\beta }^{L}$, $\mu _{\alpha }^{R}\leq \mu _{\beta }^{R}$, $\nu
_{\alpha }^{L}\geq \nu _{\beta }^{L}$, and $\nu _{\alpha }^{R}\geq \nu
_{\beta }^{R}.$ Then, we have that
\begin{align*}
S(\alpha )-S(\beta ) & =\left( \frac{\mu _{\alpha }^{L}+\mu _{\alpha }^{R}}{2%
}-\frac{\nu _{\alpha }^{L}+\nu _{\alpha }^{R}}{2}\right) -\left( \frac{\mu
_{\beta }^{L}+\mu _{\beta }^{R}}{2}-\frac{\nu _{\beta }^{L}+\nu _{\beta }^{R}%
}{2}\right) \\
& =\frac{\mu _{\alpha }^{L}-\mu _{\beta }^{L}+\mu _{\alpha }^{R}-\mu _{\beta
}^{R}+\nu _{\beta }^{L}-\nu _{\alpha }^{L}+\nu _{\beta }^{R}-\nu _{\alpha
}^{R}}{2}\leq 0.%
\end{align*}%

Case 1. Let $S(\alpha )-S(\beta )<0.$ Then, there holds $S(\alpha )<S(\beta
).$ By Definition {\ref{Def-order}, we have that }$\alpha <_{_{\mathrm{HZX}%
}}\beta.$

Case 2. Let $S(\alpha )-S(\beta )=0.$ Then, there holds $S(\alpha )=S(\beta
).$ In this case,%
\begin{equation*}
S(\alpha )=\frac{\mu _{\alpha }^{L}+\mu _{\alpha }^{R}}{2}-\frac{\nu
_{\alpha }^{L}+\nu _{\alpha }^{R}}{2}=\frac{\mu _{\beta }^{L}+\mu _{\beta
}^{R}}{2}-\frac{\nu _{\beta }^{L}+\nu _{\beta }^{R}}{2}=S(\beta ).
\end{equation*}%
This implies that
\begin{equation}
\nu _{\alpha }^{L}+\nu _{\alpha }^{R}-\nu _{\beta }^{L}-\nu _{\beta
}^{R}=\mu _{\alpha }^{L}+\mu _{\alpha }^{R}-\mu _{\beta }^{L}-\mu _{\beta
}^{R}.  \label{eqn1}
\end{equation}%
Moreover, we have that%
\begin{equation}
\begin{split}
H(\alpha )-H(\beta ) & =\left( \frac{\mu _{\alpha }^{L}+\mu _{\alpha }^{R}}{2%
}+\frac{\nu _{\alpha }^{L}+\nu _{\alpha }^{R}}{2}\right) -\left( \frac{\mu
_{\beta }^{L}+\mu _{\beta }^{R}}{2}+\frac{\nu _{\beta }^{L}+\nu _{\beta }^{R}%
}{2}\right)  \\
& =\frac{\mu _{\alpha }^{L}+\mu _{\alpha }^{R}-\mu _{\beta }^{L}-\mu _{\beta
}^{R}+\nu _{\alpha }^{L}+\nu _{\alpha }^{R}-\nu _{\beta }^{L}-\nu _{\beta
}^{R}}{2}.%
\end{split}
\label{eq2}
\end{equation}%
By the formulas \eqref{eqn1} and \eqref{eq2}, %
\begin{equation}
H(\alpha )-H(\beta )=\mu _{\alpha }^{L}+\mu _{\alpha }^{R}-\mu _{\beta
}^{L}-\mu _{\beta }^{R},  \label{eqn3}
\end{equation}%
which implies that $H(\alpha )-H(\beta)\leq 0.$\newline

Case 2.1. Let $H(\alpha )-H(\beta)<0.$ Then, there holds $H(\alpha
)<H(\beta ).$ By Definition {\ref{Def-order}, we have that }$\alpha <_{_{%
\mathrm{HZX}}}\beta.$

Case 2.2. Let $H(\alpha )-H(\beta )=0.$ By the formula~\eqref{eqn3},
\begin{equation}
\mu _{\alpha }^{L}+\mu _{\alpha }^{R}=\mu _{\beta }^{L}+\mu _{\beta }^{R}.
\label{eqn4}
\end{equation}%
This, together with the formula~\eqref{eqn1}, implies that
\begin{equation}
\nu _{\alpha }^{L}+\nu _{\alpha }^{R}=\nu _{\beta }^{L}+\nu _{\beta }^{R}.
\label{eqn5}
\end{equation}%
By the formulas~\eqref{eqn4} and \eqref{eqn5}, respectively, we obtain that
$\mu _{\alpha }^{L}-\mu _{\beta }^{L}=\mu _{\beta }^{R}-\mu _{\alpha }^{R}$
and $\nu _{\alpha }^{L}-\nu _{\beta }^{L}=\nu _{\beta }^{R}-\nu _{\alpha
}^{R}.$ Since $\mu _{\alpha }^{L}-\mu _{\beta }^{L}\leq 0$ and $\mu _{\beta
}^{R}-\mu _{\alpha }^{R}\geq 0,$ then $\mu _{\alpha }^{L}-\mu _{\beta
}^{L}=\mu _{\beta }^{R}-\mu _{\alpha }^{R}=0.$ That is, $\mu _{\alpha
}^{L}=\mu _{\beta }^{L}$ and $\mu _{\beta }^{R}=\mu _{\alpha }^{R}.$
Similarly, we have that $\nu _{\alpha }^{L}=\nu _{\beta }^{L}$ and $\nu
_{\beta }^{R}=\nu _{\alpha }^{R}.$ Hence, $\alpha =\beta .$

Therefore, we conclude that$\ \alpha \leq _{_{\mathrm{HZX}}}\beta .$
\end{proof}

\begin{remark}
Since $(\tilde{\Theta},\leq _{_{\mathrm{HZX}}})$ is a complete lattice,
 we can establish the decomposition theorem and Zadeh's extension principle
for IVIFSs as follows: for an IVIFS $A$ defined on the universe of discourse $X$,

(Decomposition Theorem) For every $x\in X$,
$$
A(x)=\bigvee\{\alpha\in \tilde{\Theta}
\mid x\in A_{\alpha}\},
$$
where $A_{\alpha}=\{z\in X \mid A(z)\geq_{_{\mathrm{HZX}}}\alpha\}$,
$\vee$ is the supremum under the linear order $\leq_{_{\mathrm{HZX}}}$.

(Zadeh's Extension Principle) Let $X$ and $Y$ be two nonempty sets and
$f: X\rightarrow Y$ be a mapping from $X$ to $Y$. Define a mapping
$\mathfrak{F}: \tilde{\Theta}^{X} \rightarrow \tilde{\Theta}^{Y}$ by
\begin{align*}
\mathfrak{F}: \tilde{\Theta}^{X} &\rightarrow \tilde{\Theta}^{Y}\\
A& \mapsto \mathfrak{F}(A)(y)=
\begin{cases}
\langle [0,0], [1, 1]\rangle, & f^{-1}(\{y\})=\varnothing, \\
\bigvee_{x\in f^{-1}(\{y\})}A(x), & f^{-1}(\{y\})\neq\varnothing,
\end{cases}
\end{align*}
which is called the {\it Zadeh's extension mapping} of $f$ in the sense of
IFSs.
\end{remark}

\section{Algebraic structures of $(\tilde{\Theta},\leq _{_{\mathrm{WLW}}})$}
\label{Sec-4}

Wang et al.~\cite{WLW2009} introduced two additional functions to
investigate the difference between two IVIFNs. In particular, they presented
the membership uncertainty index and the hesitation uncertainty index as
detailed below.

\begin{definition}[{\textrm{\protect\cite[Definition~3.3 and 3.4]{WLW2009}}}]

Let $\alpha =\langle \lbrack \mu _{\alpha }^{L},\mu _{\alpha }^{R}],[\nu
_{\alpha }^{L},\nu _{\alpha }^{R}]\rangle $ be an IVIFN. Define the \textit{%
membership uncertainty index} $T(\_)$ and the \textit{hesitation uncertainty
index} $G(\_)$ of $\alpha $ as
\begin{equation*}
T(\alpha )=(\mu _{\alpha }^{R}-\mu _{\alpha }^{L})-(\nu _{\alpha }^{R}-\nu
_{\alpha }^{L}),
\end{equation*}%
and
\begin{equation*}
G(\alpha )=(\mu _{\alpha }^{R}-\mu _{\alpha }^{L})+(\nu _{\alpha }^{R}-\nu
_{\alpha }^{L}),
\end{equation*}%
respectively.
\end{definition}

By taking a prioritized sequence of score, accuracy, membership uncertainty
index, and hesitation uncertainty index functions, the following procedure
to compare any two IVIFNs was introduced by Wang et al.~\cite{WLW2009}. This
prioritized sequence of the comparison method serves several application
fields in reality. For example, many Canadian research-intensive
institutions recruit their tenure-track faculty members following a priority
order of research first, teaching second, and service last.

\begin{definition}[{\textrm{\protect\cite[Definition 3.5]{WLW2009}}}]
\label{def.order} Let $\alpha _{1}$ and $\alpha _{2}$ be two IVIFNs. Then,
it gets the following ranking principle:
\begin{enumerate}[(1)]
\item If $S(\alpha _{1})<S(\alpha _{2})$, then $\alpha _{1}$ is smaller than
$\alpha _{2}$, denoted by $\alpha _{1}<_{_{\mathrm{WLW}}}\alpha _{2}$;

\item If $S(\alpha _{1})=S(\alpha _{2})$, then

\begin{itemize}
\item $H(\alpha _{1})<H(\alpha _{2})$, then $\alpha _{1}<_{_{\mathrm{WLW}%
}}\alpha _{2}$;

\item $H(\alpha _{1})=H(\alpha _{2})$, then

\begin{itemize}
\item $T(\alpha _{1})<T(\alpha _{2})$, then $\alpha _{1}<_{_{\mathrm{WLW}%
}}\alpha _{2}$;

\item $T(\alpha _{1})=T(\alpha _{2})$, then
\begin{itemize}
\item $G(\alpha _{1})<G(\alpha _{2})$, then $\alpha _{1}<_{_{\mathrm{WLW}}}\alpha
_{2}$;

\item $G(\alpha _{1})=G(\alpha _{2})$, then $\alpha _{1}=\alpha _{2}$.
\end{itemize}
\end{itemize}
\end{itemize}
\end{enumerate}

If $\alpha _{1}<_{_{\mathrm{WLW}}}\alpha _{2}$ or $\alpha _{1}=\alpha _{2}$,
we will denote it by $\alpha _{1}\leq _{_{\mathrm{WLW}}}\alpha _{2}$.
\end{definition}

\begin{remark}
It is easy to see that the relation $\leq _{_{\mathrm{WLW}}}$ introduction
in Definition \ref{def.order} is transitive on IVIFNs. Then, by~\cite[%
Remark~3.1 and Theorem~3.1]{WLW2009}, we get that it is a total order on
IVIFNs.
\end{remark}

In the following Theorem \ref{Completeness-1}, by applying score, accuracy,
membership uncertainty index, and hesitation uncertainty index functions, we
observe that IVIFNs in conjunction with the order $\leq _{_{\mathrm{WLW}}}$
are complete chains.

\begin{theorem}
\label{Completeness-1} $(\tilde{\Theta},\leq _{_{\mathrm{WLW}}})$ is a
complete chain.
\end{theorem}

\begin{proof}
It is similar to that of{\ Theorem~\ref{Completeness}.}
\end{proof}

\section{Concluding remarks}\label{Sec-5}

Methodologies that rank any two IVIFNs have been studied by many
researchers, such as Xu~\cite{X2007}, Ye~\cite{ye2009}, Nayagam et al.~\cite%
{N2011}, Sahin~\cite{SA2016}, Zhang and Xu~\cite{ZX2017} and Nayagam et al.
\cite{N2017}. Following this purpose, they have introduced various novel
accuracy functions. Nevertheless, their methodologies sometimes cannot
assert the difference between two IVIFNs. Then, Wang et al.~\cite{WLW2009}
and Huang et al.~\cite{HZX2021} have investigated the difference between two
IVIFNs by introducing particular additional functions. Furthermore, they
have proposed complete ranking methods for IVIFNs. The main contributions of
this study are as follows: having regard to a score function and three kinds
of entropy functions, we have shown that IVIFNs with the order in the
comparison approach for IVIFNs introduced by \cite{HZX2021} are complete
chains. Furthermore, we have observed that IVIFNs with the order in the
method for comparing IVIFNs introduced by \cite{WLW2009} are complete chains
by applying score, accuracy, membership uncertainty index, and hesitation
uncertainty index functions.

\end{document}